\pdfoutput=1
\documentclass{article}

\usepackage{blindtext} 

\usepackage[sc]{mathpazo} 
\usepackage[T1]{fontenc} 
\usepackage{microtype} 

\usepackage[english]{babel} 

\usepackage[hmarginratio=1:1,top=32mm,columnsep=20pt]{geometry} 
\usepackage[hang, small,labelfont=bf,up,textfont=it,up]{caption} 
\usepackage{booktabs} 

\usepackage{lettrine} 

\usepackage{enumitem} 
\setlist[itemize]{noitemsep} 

\usepackage{abstract} 

\usepackage{titlesec} 
\titleformat{\section}[block]{\large\scshape\centering}{\thesection.}{1em}{} 
\titleformat{\subsection}[block]{\large}{\thesubsection.}{1em}{} 

\usepackage{fancyhdr} 
\usepackage{titling} 
\usepackage{hyperref} 
\hypersetup{
    colorlinks=true,
    linkcolor=blue,
    filecolor=magenta,      
    urlcolor=cyan,
}
\usepackage{xcolor}
\usepackage{tikz}
\usepackage{graphicx}
\usepackage{caption}
\usepackage{subcaption}
\usepackage{float}
\usepackage{etoolbox}
\usepackage{graphicx}
\usepackage{mathtools}
\usepackage{amsmath}
\usepackage{pgfplots}
\usepackage{amsmath}
\usepackage{tikz}
\usepackage{mathtools}
\usepackage{amsmath}
\usepackage{amsfonts}
\usepackage{amssymb}
\usepackage{amsthm}
\usepackage{commath}

\newtheorem{thm}{Theorem}
\theoremstyle{definition}

\newtheorem{lemma}{Lemma}

\usetikzlibrary{positioning,calc}
\usepackage{setspace}

\pretitle{\begin{center}\huge\bfseries} 
\posttitle{\end{center}} 
\title{Bio-plausible Unsupervised Delay Learning for Extracting Temporal Features in Spiking Neural Networks} 
\author{%
\textsc{Alireza Nadafian} \\[1ex] 
\normalsize University of Tehran \\ 
\normalsize \href{mailto:alirezanadafian@ut.ac.ir}{alirezanadafian@ut.ac.ir} 
\and 
\textsc{Mohammad Ganjtabesh} \\[1ex] 
\normalsize University of Tehran \\ 
\normalsize \href{mailto:mgtabesh@ut.ac.ir}{mgtabesh@ut.ac.ir} 
}
\date{} 
\renewcommand{\maketitlehookd}{%
\begin{abstract}
\noindent The plasticity of the conduction delay between neurons plays a fundamental role in learning. However, the exact underlying mechanisms in the brain for this modulation is still an open problem. Understanding the precise adjustment of synaptic delays could help us in developing effective brain-inspired computational models in providing aligned insights with the experimental evidence. In this paper, we propose an unsupervised biologically plausible learning rule for adjusting the synaptic delays in spiking neural networks. Then, we provided some mathematical proofs to show that our learning rule gives a neuron the ability to learn repeating spatio-temporal patterns. Furthermore, the experimental results of applying an STDP-based spiking neural network equipped with our proposed delay learning rule on Random Dot Kinematogram indicate the efficacy of the proposed delay learning rule in extracting temporal features.   \end{abstract}
}

\begin{document}

\maketitle


\section{Introduction}

Devising computational models inspired from the brain are becoming more attractive every day. The reason for that is the ability of these kinds of models to mimic the underlying brain mechanisms in different tasks. Also, from the artificial intelligence (AI) and machine learning (ML) points of view, they can be adopted for solving many real life problems and widespread applications. Although there are a lot of improvements in these models, we are still far from reaching an accurate model for simulating the brain. This is mostly related to the lack of experimental data about the underlying mechanisms in the brain, where developing mathematical or computational models would be of great importance.

Artificial Neural Networks (ANNs) are one category of these brain-inspired models. The origin of these models goes back to the works of McCulloch and Pitts that presented a simple computational model for a neuron with logical circuits \cite{mcculloch:1943}. Their proposed neuron model receives some input signals and provides a spike when the sum of them become greater than a specified threshold. By composing these artificial neurons, connected by synapses, we can get an ANN that applies a transformation on the inputs and it can be used for a wide range of tasks. The core feature of almost all kind of ANNs is their ability to learn. Learning is an exceptional feature that can be described by  chemical processes in the brain which results in modifying the connections strength between neurons. The ability of learning is the basis of every complex and high-level cognitive function of our brain, including reasoning, language, memory and decision making. 

In 1949, Donald Hebb introduced a theory for the learning of the synaptic strength that claims \cite{hebb:1949}: "The neurons that fire together wire together". This theory became the basis of many learning rules such as Spike-Timing Dependant Plasticity (STDP) \cite{stdp:2008}. It has been shown that STDP contributes to set proper synaptic weights for learning the repeating patterns \cite{stdp_rep:2001} in an unsupervised manner, therefore it has been used in many ANNs with different structures such as Spiking Neural Networks (SNNs) for various tasks.

Delays in sending or receiving information through axons and synapses have an important role in the functional behavior of the brain and many of our daily activities. The ability to walk, play a musical instrument, and vision are all dependent on the conduction delay between the neurons. Besides, many computational studies emphasize the need for conduction delays in providing some stable and reproducible firing patterns. An example is the polychronization concept that was proposed by Izhikevich \cite{polychron}. He indicated that the conduction delay between the neurons is an important property for the creation of the polychronous groups which can explain the high-level functions of the brain.

Synaptic delay plasticity plays an important role in learning and understanding spatio-temporal patterns in the brain \cite{myelin_p}. Some evidences suggest that the underlying process of synaptic delay plasticity, known as myelination, can be influenced by action potentials and neuronal activities \cite{myelin_activity1, myelin_activity2}. This activity-dependent property of myelination can contribute to the learning of repeating temporal patterns in many time-dependent tasks \cite{myelin_piano}. It has been shown that these delays could be adjusted with respect to the pattern of neural impulses, again in an unsupervised manner \cite{delay_change:1998}. Despite the important role of the synaptic delays, most of the computational models used to ignore them to avoid extra delay-related complexity. However, ignoring the delays makes these computational models ineffective in providing aligned insights with the experimental evidences. Besides the delays, devising a rule for adjusting them is an important problem. Recent studies suggest some supervised gradient based delay learning rules, using multiple synaptic delays, and unsupervised rule using EM algorithm \cite{myelin_recent1, myelin_recent2, myelin_recent3, myelin_recent4}. However, all the mentioned learning rules are not biologically plausible.

In this paper, we propose an unsupervised learning rule for adjusting the synaptic delays which only depends on the spike timing of pre- and post-synaptic neurons. After that, we present some mathematical proofs to show the ability of our proposed rule in learning the spatio-temporal patterns. Finally, we apply our delay learning rule on an STDP-based spiking neural network in an experiment to investigate its ability to learn repeating spatio-temporal patterns.


\section{Materials and Methods}
The experimental evidence has shown that the repetitive activation of two connected neurons can induce long term changes in the synaptic connection between them, which is considered as the basis of the learning related processes in the brain \cite{bio_stdp}. Spike-Timing Dependant Plasticity (STDP) is a biologically plausible learning rule that contributes to the adjustment of synaptic strength between two connected neurons. If a post-synaptic neuron fires at a specific time, then STDP makes the synapses of the pre-synaptic neurons that fired before stronger, and those fired after, weaker. This is the reason for the selective property of STDP and somehow justifies its ability to learn repeating patterns \cite{stdp_rep:2001}. Normally, the synaptic delay ($d$) is not considered in firing time difference ($\Delta t$) of typical STDP, but in the case of existing delays over synapses, we need to take it into account so that the STDP correctly adjusts the synaptic weights by considering the effect of the delay on the spike timing of the post-synaptic neuron. Therefore, for two connected neurons $i$ and $j$, the STDP adjusts their synaptic weight $w_{i,j}$ with respect to the firing time difference as well as the current synaptic delay $(d_{i,j})$. The equations for delay-related STDP are as follow:
\begin{align}
\begin{aligned}
    \nonumber \Delta t_{i, j} =& t_\text{j} - t_\text{i}- d_{i, j},\text{ and} \\
    \Delta w_{i, j} = F(\Delta t_{i, j}) =& \begin{cases}
    A_+ \exp{\frac{-\Delta t_{i, j} }{\tau _+}} & \text{ $\Delta t_{i, j}  \geq 0$}, \\ 
    -A_- \exp{\frac{\Delta t_{i, j} }{\tau _-}} & \text{ $\Delta t_{i, j}  < 0$},
    \end{cases}
\end{aligned}
\end{align}
where $A_+$ and $A_-$ are the parameters that control the amount of synaptic potentiation and depression, respectively. Also, $\tau _+$ and $\tau _-$  are the time windows of STDP, respectively for synaptic potentiation and depression. The schematic representation of delay-related STDP is provided in Fig. \ref{fig:1}. 

\begin{figure}[h]
       \begin{subfigure}{.5\textwidth}
        \centering
         \begin{tikzpicture}
          \draw[->] (-2.75,0) -- (2.75,0) node[right] {$\Delta t_{i, j}$};
          \draw[->] (0,-2) -- (0,2) node[above] {$\Delta w_{i, j}$};
          \draw[dashed][scale=0.5,domain=-5:0,smooth,variable=\x] plot ({\x},{-3* 2.7 ^ (\x / 1.5)});
          \draw[dashed][scale=0.5,domain=0:5,smooth,variable=\x] plot ({\x},{3* 2.7 ^ (-\x / 1.5)});
          \draw[fill=white] (0,-1.5) circle (2pt) (0,-1.5) node[anchor=west] {$A_{-}$};
          \filldraw[black] (0,1.5) circle (2pt) node[anchor=east] {$A_{+}$};
        \end{tikzpicture}

    \end{subfigure}
    \begin{subfigure}{0.5\textwidth}
        \centering
                        \begin{tikzpicture}[
                roundnode/.style={circle, draw=black!60, very thick, minimum size=7mm}]
            \node[roundnode, label ={[align=left]below: pre-synaptic\\neuron}](pre){i};
            \node[roundnode, label = {[align=left]below: post-synaptic\\neuron}](post)[right=2cm of pre] {j};
            
            \draw[bend left,->]  (pre) to node [auto] {$w_{i, j}$, $d_{i, j}$} (post);

            \end{tikzpicture}
         \end{subfigure}
        \caption{Schematic representation of delay-related STDP function.}
        \label{fig:1}

\end{figure}
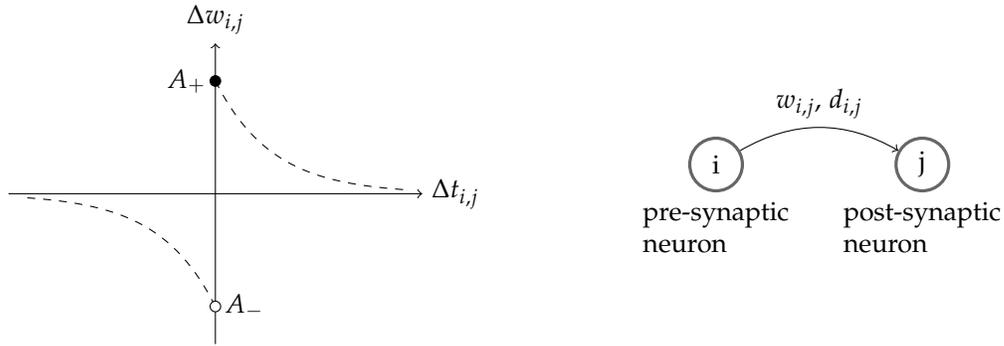

\subsection{Delay Learning}
It has been shown that the conduction delay of information between two connected neurons can be changed during the time as a result of myelination process \cite{myelination}. The main purpose of myelination process is to further enhance the conduction speed over the synapse \cite{myelin}. The dynamic nature of the myelination process in the brain has an important role in synaptic plasticity and learning spatio-temporal features \cite{myelin_p}. 

The main question is how we could adjust the synaptic delays for extracting the appropriate spatio-temporal features? 
To answer this question and devise a new learning rule for adjusting synaptic delays, we claim that the synaptic delay should be decreased (increased) as the synaptic weight is potentiated (depressed). The biological evidence for this claim is as follows \cite{myelin_piano, myelin_speed}. When we start learning a new skill (say, driving), at the very beginning stages, our reactions take longer time to be performed and they become faster as we continue repeating the task. Part of this speed-up is related to the potentiation of appropriate synapses that increase the synaptic efficacy of an appropriate neuron on the corresponding post-synaptic neuron that makes it to fire sooner.
But, at the same time, synaptic delays are decreased for those potentiated synapses, making the post-synaptic neuron to fire even sooner.

In this paper, we consider conduction delay as a floating point number for each synapse in spiking neural networks. Our proposed rule is somehow similar to the STDP learning rule but in negative polarization. Therefore, it locally depends on the firing times of two connected neurons in an unsupervised manner.  The equations for our proposed delay learning rule are as follow:
\begin{align}
\begin{aligned}
    \nonumber  \Delta t_{i,j} =& t_{j} - t_{i}- d_{i,j},\text{ and} \\
 \Delta d_{i,j} = G(\Delta t_{i,j}) =& \begin{cases}
 -B_- \exp{\frac{-\Delta t_{i,j} }{\sigma_-}} & \text{ $\Delta t_{i,j}  \geq 0$}, \\ 
 B_+ \exp{\frac{\Delta t_{i,j} }{\sigma_+}} & \text{ $\Delta t_{i,j}  < 0$}, 
 \end{cases}
\end{aligned}
\end{align}
where  $B_+$ and  $B_-$ are the parameters that control the value of synaptic delay reduction and increament, respectively. Also, $\sigma_+$ is the time constant of delay learning rule in the case of the pre-synaptic neuron causal effect on the firing of post-synaptic neuron and $\sigma_-$ is the time constant for the case of independency of the firings of two connected neurons.
 The schematic representation of our proposed delay learning rule  is provided in Fig. \ref{fig:2}.

\begin{figure}[h]
   \begin{subfigure}{.5\textwidth}
    \centering
         \begin{tikzpicture}
              \draw[->] (-3,0) -- (3,0) node[right] {$\Delta t_{i, j}$};
              \draw[->] (0,-2.5) -- (0,2.5) node[above] {$\Delta d_{i, j}$};
              \draw[dashed][scale=0.5,domain=-6:0,smooth,variable=\x] plot ({\x},{3* 2.7 ^ (\x / 1.5)});
              \draw[dashed][scale=0.5,domain=0:6,smooth,variable=\x] plot ({\x},{-3* 2.7 ^ (-\x / 1.5)});
              \draw [fill=white] (0,1.5) circle (2pt) node[anchor=west] {$B_{+}$};
              \filldraw[black] (0,-1.5)  circle (2pt) node[anchor=east] {$B_{-}$};
        \end{tikzpicture}
    \end{subfigure}
    \begin{subfigure}{0.5\textwidth}
        \centering
        \begin{tikzpicture}[
                roundnode/.style={circle, draw=black!60, very thick, minimum size=7mm}]
            \node[roundnode, label ={[align=left]below: pre-synaptic\\neuron}](pre){i};
            \node[roundnode, label = {[align=left]below: post-synaptic\\neuron}](post)[right=2cm of pre] {j};
            
            \draw[bend left,->]  (pre) to node [auto] {$w_{i, j}$, $d_{i, j}$} (post);

            \end{tikzpicture}
         \end{subfigure}
    \caption{Schematic representation of our proposed delay learning function.}
    \label{fig:2}
\end{figure}
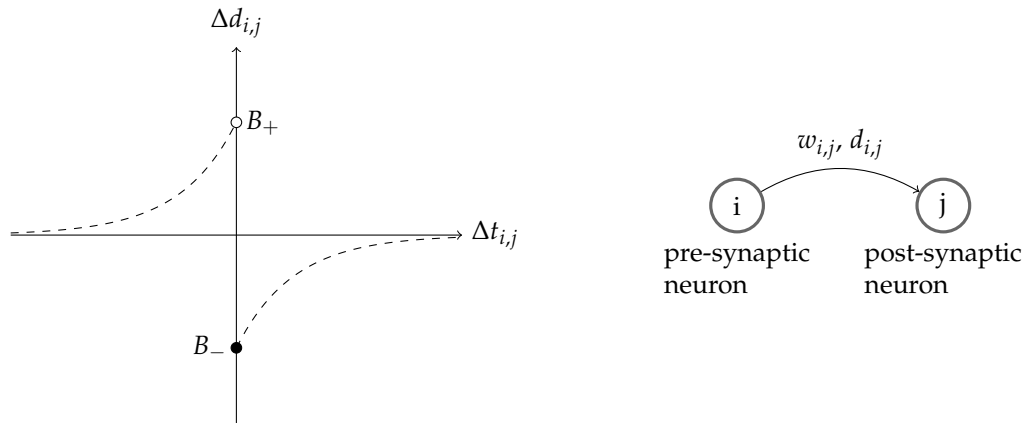

Now, from the mathematical point of view, we illustrate how our proposed rule contributes to learning the spatio-temporal features. In an ideal situation, where the synaptic weights and delays of a post-synaptic neuron perfectly match to a spatio-temporal feature, it is expected that  the time difference in Fig. \ref{fig:2}  becomes zero for all active pre-synaptic neurons. In other words, all the delays are adjusted perfectly in such a way that the post-synaptic neuron receives all the spikes of its pre-synaptic neurons at the same time, causing the post-synaptic neuron to be fired, according to delays.

For a specific post-synaptic neuron, say $j$, let $PRE_j$ denotes the set of its pre-synaptic neurons, i.e. $PRE_j = \{i|\text{ } i$ \text{is a pre-synaptic neuron connected to neuron }$j\}$. For each neuron $i\in PRE_j$, let $w_{i,j}$ and $d_{i,j}$ denote the synaptic weight and delay between neuron $i$ and $j$, respectively. Now, suppose that neuron $j$ emits a spike at time $t_j$ due to a specific spike pattern $P$ and let $$\Delta t_{i,j} = t_j - t_i - d_{i,j},$$ where $t_i$ is the firing time of pre-synaptic neuron $i$. Also, let $A_j(t_j)$ denotes those active pre-synaptic neurons that contribute in the firing of neuron $j$ $(A_j(t_j) \subseteq PRE_j)$. Obviously, for all neurons in $A_j(t_j)$, we have $\Delta t_{i,j} \geq 0$.

At the firing time of neuron $j$, there is at least a pre-synaptic active neuron, say $l \in A_j(t_j)$, that causes the potential of neuron $j$ to reach the threshold and makes it to fire (i.e. $\Delta t_{l,j} = 0$). In other words, before the firing of neuron $j$, it receives the spike of neuron $l$ later than the other pre-synaptic active neurons. First, we show that the amount of delay change is no more than $\Delta t_{i,j}$ for any pre-synaptic active neuron $i$ with $\Delta t_{i,j} \geq 0 $.

\begin{lemma}
\label{lemma1}
Suppose that  $0 < B_- \leq \sigma_-$. For any $i \in A_j(t_j)$ with $ \Delta t_{i,j} \geq 0  $, we have $$B_- - B_- \exp{\frac{-\Delta t_{i,j}}{\sigma_-}} \leq \Delta t_{i,j} .$$
\end{lemma}
\begin{proof} Let $f(\Delta t_{i,j}) = \Delta t_{i,j} - B_- +B_- \exp{\frac{-\Delta t_{i,j}}{\sigma_-}}$. Then, we have $f(0) = 0$ and $\frac{d f}{d \Delta t_{i,j}} \geq 0$, for $\Delta t_{i,j} \geq 0$, since:
\begin{align*}
    \frac{df}{d\Delta t_{i,j}} = 1 - \frac{B_-}{\sigma_-}exp(\frac{-\Delta t_{i,j}}{\sigma_-}) \geq 0 
\end{align*}
Therefore, for $ \Delta t_{i,j} \geq 0  $, we have $f(\Delta t_{i,j}) \geq 0$.

\end{proof}

According to Lemma \ref{lemma1}, if $ \Delta t_{i,j}$ is positive before updating the corresponding synaptic delay, it remains positive afterward. Now, we show that for a repeating pattern, say $P$, the last active neuron $l \in A_j(t_j)$ remains last after updating the synaptic delays.

\begin{lemma}
\label{lemma2}
For a repeating pattern $P$, suppose that $\Delta t^{old}_{l,j} = 0 $ for an active pre-synaptic neuron $l$. After updating the synaptic delays, we have $\Delta t_{l,j}^{new} = 0 $.
\end{lemma}
\begin{proof}
It is sufficient to show that $\forall i \in A_j(t_j)$,  $\Delta t_{i,j}^{new} \geq \Delta t_{l,j}^{new}$. To do this, we have:
\begin{align*}
 \Delta t_{i,j}^{new} - \Delta t_{l,j}^{new} =& t_j^{new}-t_i-d_{i,j}^{new} - t_j^{new}+t_l-d_{l,j}^{new} \\ =& t_l - t_i + d_{l,j}^{new} + d_{i,j}^{new} \\ 
 =& t_l - t_i + d_{l,j}^{old} -B_- - d_{i,j}^{old} + B_- \exp{\frac{-\Delta t_{i,j}^{old}}{\sigma_-}} \\
 =& t_j^{old}-t_i-d_{i,j}^{old}-B_- - t_j^{old}+t_l-d_{l,j}^{old}+ B_- \exp{\frac{-\Delta t_{i,j}^{old}}{\sigma_-}} \\
 =& \Delta t_{i,j}^{old} - \Delta t_{l,j}^{old} -B_- + B_- \exp{\frac{-\Delta t_{i,j}^{old}}{\sigma_-}} \\ 
=& \Delta t_{i,j}^{old} -B_- + B_- \exp{\frac{-\Delta t_{i,j}^{old}}{\sigma_-}} \geq 0 \\
\Rightarrow&\Delta t_{i,j}^{new} \geq \Delta t_{l,j}^{new}
\end{align*}
\end{proof}
Since the last neuron $l$ remains last after updating the synaptic delays, therefore we have: 
\begin{align*}
    &\Delta t_{l,j}^{new} - \Delta t_{l,j}^{old} = 0 \\ 
    \Rightarrow & t_j^{new} - t_l - d_{l,j}^{new} - t_j^{old} + t_l + d_{l,j}^{old} = 0 \\ 
    \Rightarrow & t_j^{new} - t_j^{old} = d_{l,j}^{new} - d_{l,j}^{old}
\end{align*}
As $\Delta t_{l,j}^{old} = 0 $, we have $d_{l,j}^{new} - d_{l,j}^{old} = -B_-$ (according to the delay learning rule). So we have: $$t_j^{new} - t_j^{old} = -B_-.$$

This means that for a repeating pattern P, the firing time of the post-synaptic neuron $j$ is shifted by an amount of $-B_-$, i.e. it fires sooner in the next repetition of pattern $P$. According to the synaptic delay learning rule, decreasing delay is done not only for the last neuron, but also for the other pre-synaptic active neurons. This is proved in the following theorem.

\begin{thm}
For all $i \in A_j(t_j)$, $0 \leq \Delta t_{i,j}^{new} \leq \Delta t_{i,j}^{old}$.
\end{thm}
\begin{proof}
According to Lemma \ref{lemma2}, we have $0 = \Delta t_{l,j}^{new} \leq \Delta t_{i,j}^{new}$.
So, it is sufficient to show that $\Delta t_{i,j}^{new} \leq \Delta t_{i,j}^{old}$.
To do this, we have :
\begin{align*}
    \Delta t_{i,j}^{new} - \Delta t_{i,j}^{old} &= t_j^{new}- t_i - d_{i,j}^{new} - t_j^{old} + t_i + d_{i,j}^{old}\\ &=t_j^{new} - t_j^{old}+d_{i,j}^{old} - d_{i,j}^{new} \\ 
    &= - B_- + B_-\exp{\frac{-\Delta t_{i,j}^{old}}{\sigma_-}} \leq 0
\end{align*}
\begin{equation*}
    \Rightarrow \Delta t_{i,j}^{new} \leq \Delta t_{i,j}^{old}
\end{equation*}
\end{proof}

This theorem indicates that by repeating the same pattern $P$, the synaptic delays are adjusted in such a way that the temporal difference between the first and last pre-synaptic active neurons decreased. Let $\Delta t_{i,j}^{n}$ denotes the value of $\Delta t_{i,j}$ at step $n$ of applying the delay learning rule for a repeating pattern $P$.

\begin{lemma}
\label{lemma3}
For a repeating pattern $P$ and all $i \in A_j(t_j)$, $\lim_{n\to\infty} \Delta t_{i,j}^n = 0$.
\end{lemma}
\begin{proof}
With respect to the previous theorem, we have $0 \leq \Delta t_{i,j}^{k+1} \leq \Delta t_{i,j}^{k}$. So, the sequence $\Delta t_{i,j}^{k}$ is bounded decreasing and therefore convergent. Now suppose that $\lim_{n\to\infty} \Delta t_{i,j}^n = L$. It is sufficient to show that $L = 0$. To this end, we do as follows:
\begin{align*}
    &\Delta t_{i,j}^{k+1} = \Delta t_{i,j}^k - B_- + B_-\exp{\frac{-\Delta t_{i,j}^k}{\sigma_-}} \\
    {\Longrightarrow} &\lim_{k\to\infty}  \Delta t_{i,j}^{k+1} = \lim_{k\to\infty} \Delta t_{i,j}^{k} - B_- + B_-\exp(\frac{-\lim_{k\to\infty} \Delta t_{i,j}^{k}}{\sigma_-}) \\
    \Longrightarrow &L = L - B_- + B_-\exp(\frac{-L}{\sigma_-})
    \Longrightarrow \exp{\frac{-L}{\sigma_-}} = 1 \Rightarrow L = 0.
\end{align*}
\end{proof}

Finally, when the pre-synaptic neuron $i$ doesn't contribute to the firing of the post-synaptic neuron $j$, we have $\Delta t_{i,j} < 0$. In this case, the delay learning rule increases the synaptic delay and at the same time, $t_j$ is decreased by an amount of $B_-$ for the next repetition of pattern $P$. Therefore, by repeating the same pattern $P$, we have: 
\begin{align*}
    \Delta t_{i,j}^{new} - \Delta t_{i,j}^{old} &= t_j^{new}- t_i - d_{i,j}^{new} - t_j^{old} + t_i + d_{i,j}^{old}\\ &=t_j^{new} - t_j^{old}+d_{i,j}^{old} - d_{i,j}^{new} \\ 
    &= - B_- - B_+\exp{\frac{\Delta t_{i,j}^{old}}{\sigma_+}} \leq 0 \\
     & \Longrightarrow \Delta t_{i,j}^{new} \leq \Delta t_{i,j}^{old}.
\end{align*} 

In other words, by applying the synaptic delay learning rule on the synapse connecting neurons $i$ and $j$, it is less likely that the neuron $i$ contributes to the firing of neuron $j$, in the next repetition of pattern P.

\subsection{Stop Condition for Delay Learning}
The only issue with our proposed delay learning rule is the convergence of all the delays to zero. This means that by repeating a pattern, all the delays corresponding to the pre-synaptic active neurons  keep decreasing until all of them become zero, which results in losing the temporal feature of the learnt pattern. Therefore, a stop condition is required to avoid converging delays to zero (hyper-myelination). The biological evidence shows that we have the same property in myelination process, and this deactivation of myelination process concerns multiple axons \cite{brake}.

For a post-synaptic neuron $j$, we stop the delay learning on $d_{i,j}$ for every $i \in PRE_j$ if:
$$ \exists k \in PRE_j, \; d_{k,j} < c,$$ where $c>B_-$ is a constant.
As a result, no synaptic delay becomes zero and hence the learnt temporal features are kept due to this stop condition.

Finally, we proved in Lemma \ref{lemma3} that after infinite steps of performing our delay learning rule on the synaptic delays, the post-synaptic neuron will be completely fitted on any temporal repeating pattern, but according to the above mentioned stop condition, the delay learning process may terminate before the convergence. To overcome this issue, we add another modulation to the learning process, in which we increase all the synaptic delays by a constant small value at every time step. Note that this modulation doesn't violate any part of the previous theorems and lemmas,  as well as their proofs.

\subsection{Homeostasis Regulation}
Homeostatic synaptic plasticity has been observed experimentally and it may serve to stabilize plasticity mechanisms and regulate the behaviour of the network \cite{homeo}. When dealing with synaptic weights, the homeostasis plasticity helps the neurons to forget their incorrectly learnt spatial patterns and lets them to learn new ones. It also helps to increase the representation diversity among the neurons by preventing a single neuron to  learn multiple patterns. In addition, it adjusts the weights to regularize the firing rate of neurons in the case of low or high average initial weights. To stabilize STDP learning rule, the homeostasis mechanism could be applied as follows:
$$K = \frac{R_{target} - R_{observed}}{R_{target}},\text{ and}$$
$$w_{new} = w_{current} + \lambda_w \times K , $$
where $R_{target}$ and $R_{observed}$ are the expected and observed firing rates of the post-synaptic neuron, respectively, and $\lambda_w$ is the learning rate of homeostasis mechanism. According to this rule, in the case of $R_{observed} > R_{target}$, all the synaptic weights of a post-synaptic neuron are decreased by a constant to make the neuron to be less active and in the case of $R_{target} > R_{observed}$, the synaptic weights are increased to make the neuron to be more active.

Accordingly, similar mechanism could be applied for the synaptic delay learning rule. So, we define the homeostasis mechanism for synaptic delays as follows. 
$$d_{new} = d_{current} - \lambda_d \times K , $$
where $\lambda_d$ is the learning rate corresponding to the homeostasis mechanism for synaptic delays. According to this rule, in the case of $R_{observed} > R_{target}$, all the synaptic delays of a post-synaptic neuron are increased by a constant to make the post-synaptic neuron to fire later than the other neurons and in the case of $R_{target} > R_{observed}$, the synaptic delays are decreased to make it fire sooner than the other neurons.

\section{Results}

In this section, we describe the result of applying our delay learning rule on a spiking neural network in an experimental task to show its ability to learn repeating spatio-temporal patterns. First, we explain the architecture of the employed spiking neural network and the learning and regulation mechanisms used in it. Then, we describe the employed task on this model and illustrate how this model can learn repeating spatio-temporal patterns with the aim of the proposed delay learning rule.

\subsection{Network Architecture}
To show the ability of our delay learning rule, we applied it on a feed forward spiking neural network with two layers. The first layer is a two dimensional array of neurons, each encodes a specific position in the input stimulus. The second one is a convolutional layer with $5\times5$ kernel size and one unit stride in each direction. This layer has four shared-weight (and shared-delay) features. The overall structure of the network is provided in Fig. \ref{fig:a}.

The Leaky Integrate and Fire (LIF) model \cite{LIF} is utilized in the convolutional layer, where a random noise on the membrane potential is also employed. These neurons receive the input signals with respect to the adjusted synaptic delays. Hence, the leakage term of the LIF neuronal model plays an important role in the perception of a temporal pattern. Neurons can emit only one spike during the presentation of each stimulus (first spike coding). Also, neurons have a simple linear adaptation on the threshold value, i.e. the threshold value is increased when the neuron fires and it is decreased otherwise. This adaptation is required to avoid isolation of a neuron due to the initial random weights and delays and it also contributes to forget incorrectly learnt features. Moreover, the lateral inhibition mechanism is employed between the neurons in different feature maps, but at the same location, i.e. each time a neuron fires, it inhibits the neurons at the same location but in different feature maps. The reason for this lateral inhibition is to avoid the same pattern to be learnt by multiple neurons in different feature maps.

\begin{figure}[!h]
\centering
  \includegraphics[width=0.5\linewidth]{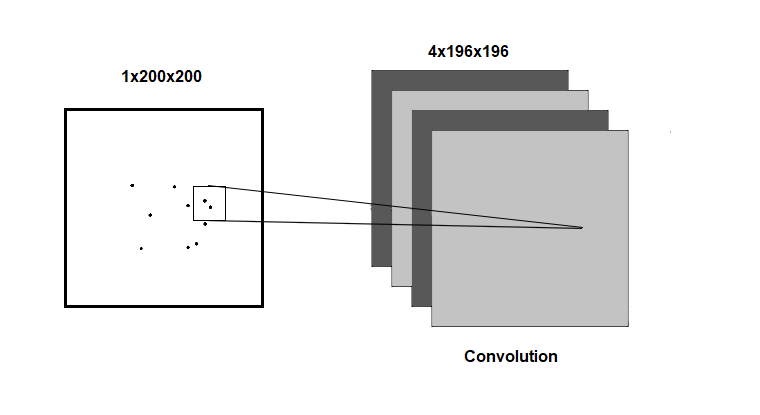}
  \caption{Structure of the network}
  \label{fig:a}
\end{figure}

\subsection{Experiment}
To provide an experimental evidence on the ability of the proposed delay learning rule, we employed the random moving dots task in which the model has to learn the temporal patterns. In this task, there are some moving dots on a two dimensional plane, where some of them move toward a specific direction and the others move randomly. Thereby, our model has to learn the repeating patterns, which is the specific directions that the dots moving toward. 

\begin{figure}[!h]
     \centering
     \begin{subfigure}[b]{0.24\textwidth}
         \centering
         \includegraphics[width=\textwidth]{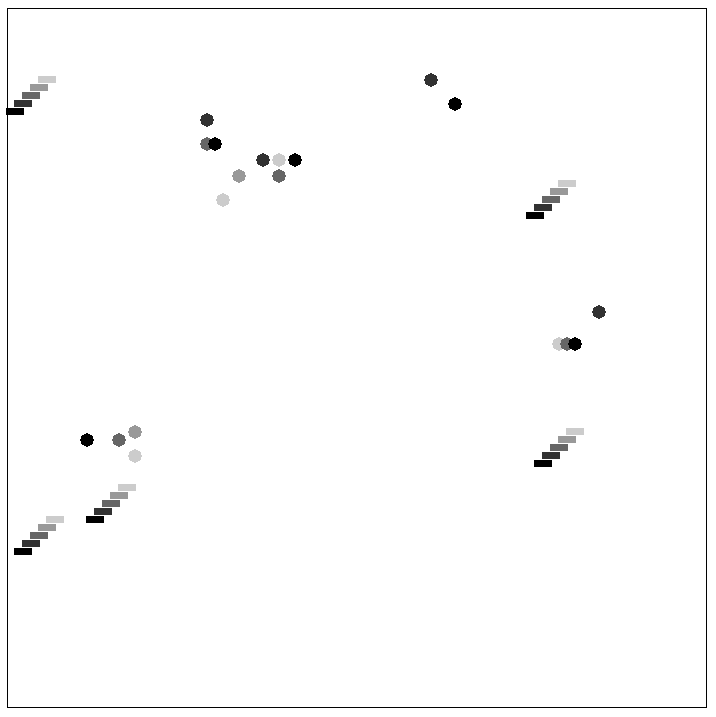}
         \caption{$-135^\circ$ angle}
     \end{subfigure}
     \begin{subfigure}[b]{0.24\textwidth}
         \centering
         \includegraphics[width=\textwidth]{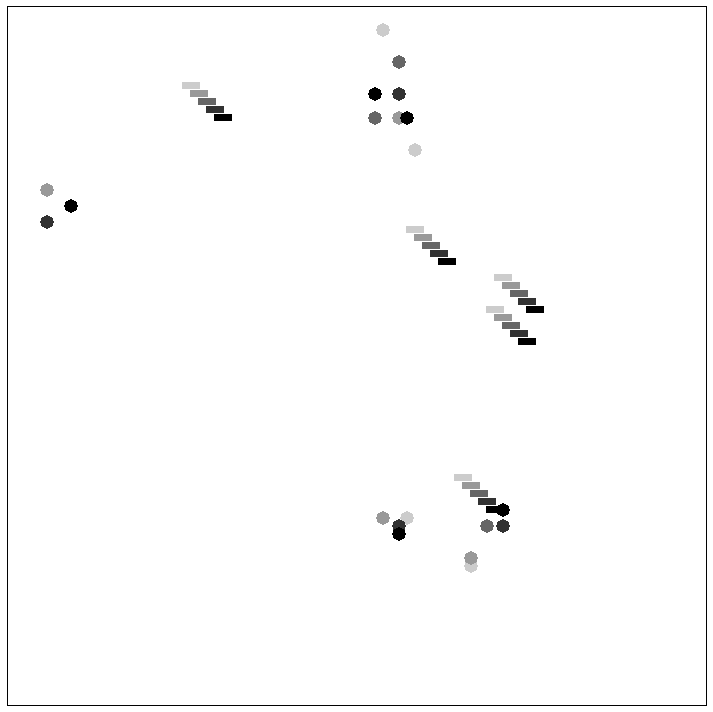}
         \caption{$-45^\circ$ angle}
     \end{subfigure}
     \begin{subfigure}[b]{0.24\textwidth}
         \centering
         \includegraphics[width=\textwidth]{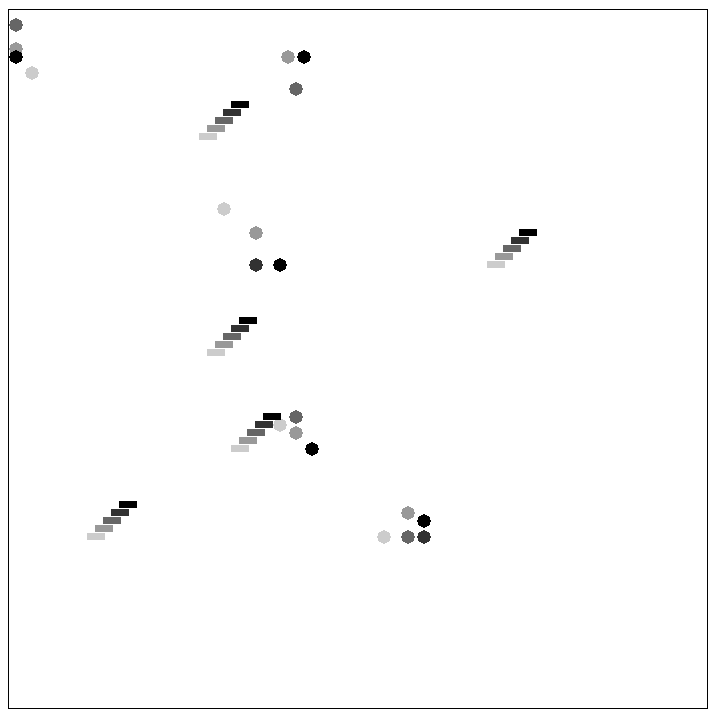}
         \caption{$45^\circ$ angle}
     \end{subfigure}
     \begin{subfigure}[b]{0.24\textwidth}
         \centering
         \includegraphics[width=\textwidth]{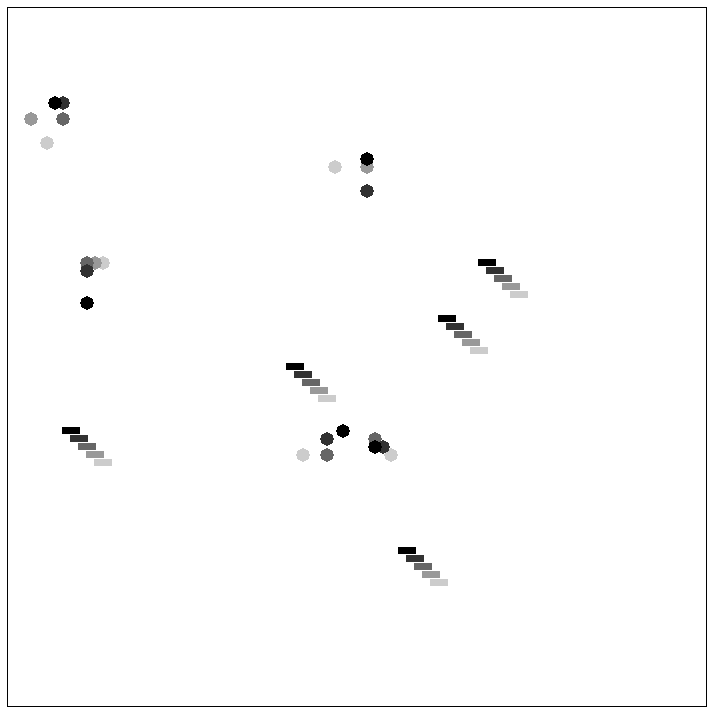}
         \caption{$135^\circ$ angle}
     \end{subfigure}
      \caption{Four samples of the employed dataset. The intensity of each dot shows the location of it over the time (black indicates the most recent location). Plus, the trajectory of each dot illustrates whether it moves randomly or not (circles are the dots that move randomly while the rectangles are those moving toward one of the four diagonal directions).}
        \label{fig:inp}
\end{figure}

The dataset of this experiement was generated by Random Dot Kinematogram of the \emph{PsychoPy} library \cite{psychopy}. This dataset consists of 100 stimuli, where in each stimulus, 10 dots move during five time steps with one time unit gap between successive steps. In each stimulus, half of the dots move toward one of the four diagonal directions ($\pm 45^\circ$ and $\pm 135^\circ$) at a fixed speed and the other half move randomly with different speeds. We expect that our model learn these repeating four types of  movement toward diagonal directions. Four samples of this dataset are provided in Fig. \ref{fig:inp}.

\subsection{Training Phase}
As it was mentioned before, the model has four shared-weight (and shared-delay) features in the convolutional layer. Training the model is continued until all of the features get frozen according to our delay learning stop condition, where the whole training phase took 33 epochs. In each epoch, we presented the whole generated dataset containing 100 stimuli. Updating the synaptic delays right after the firing of a post-synaptic neuron can be misleading for the neurons on the same feature map, since changing their synaptic delays can delude the timing of the current information passing through them. To solve this issue, we tag all neural activities and apply our delay learning rule on the synaptic delays at the end of the stimulus presentation. In addition, before presenting a stimulus to the model, we reset the states of all neurons to their predefined parameter values as they are depicted in Table \ref{table:params}. The learning process of the model is illustrated in Fig. \ref{fig:le} for several epochs. In this figure, the size of the circles represent the synaptic delays in the corresponding feature map and their intensities indicate the synaptic weights. As it can be seen in Fig. \ref{fig:le}, each feature map learnt a direction of moving toward one of the four diagonal directions over the time. 
\begin{table}[h!]
\centering
\begin{tabular}{||c | c  | c ||} 
 \hline
Modulation & Parameter  & Value \\ [0.5ex] 
 \hline\hline
  Neurons & Threshold & 4.15   \\ 
   & $\tau_m$ & 20   \\ 
 \hline
  & $A_+$ & 5.0   \\ 
  & $A_-$ & 5.0   \\ 
  STDP & $\tau_+$ & 0.0001 \\
  & $\tau_+$ & 0.0001 \\
  & Weights & $N(0.95 , 0.05)$ \\
 \hline
 
\end{tabular}
\begin{tabular}{||c | c  | c ||} 
 \hline
Modulation & Parameter  & Value \\ [0.5ex] 
 \hline\hline
  Delay Regulation & $c$ & 0.001 \\
  & Growth factor  & 0.0001 \\
 \hline
  & $B_+$ & 5.0   \\ 
  & $B_-$ & 5.0   \\ 
  Delay learning & $\sigma_+$ & 0.001 \\
  & $\sigma_+$ & 0.001 \\
  & Delays & $N(50, 0.02)$ \\
 \hline
 
\end{tabular}
\caption{Parameters of the model and their predefined values.}
\label{table:params}
\end{table}

The results of this experiment indicate that the proposed delay learning rule could help the model to learn the temporal features (moving dots) in an unsupervised manner. Although the employed task is very simple, but it shows the potential ability of our proposed delay learning rule to extract more complex temporal features.

\begin{figure}[!h]
\centering
  \includegraphics[width=1\linewidth]{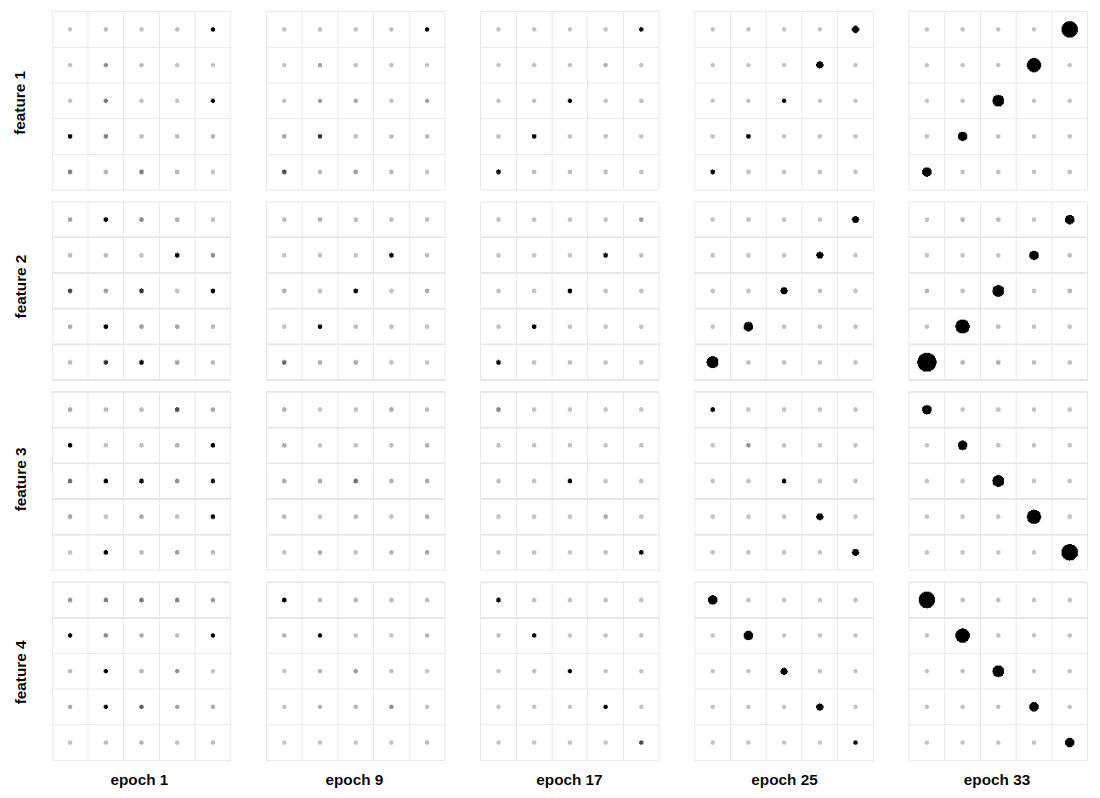}
  \caption{The learnt spatio-temporal features over the epochs. The size (intensity) of each circle represents the delay (weight) of its corresponding synapse, where the bigger circle illustrates lower synaptic delay (black indicates the higher synaptic weight).}
  \label{fig:le}
\end{figure}

\section*{Discussion}
The modulation of conduction delay between neurons is an essential ability of the brain in the learning of temporal patterns. However, the exact underlying mechanisms in the brain for the precise adjustment of the synaptic delays remains an open problem. In this paper, we introduced a biologically plausible learning rule that provides an answer to this question from a computational point of view. To support our solution, we provided some mathematical proofs showing that a post-synaptic neuron, receiving a repeating spatio-temporal pattern, can precisely adjust its synaptic delays so that it responds faster and more accurately in the next repetition of the pattern. Then, we show that an STDP-based spiking neural network equipped with our proposed delay learning rule could effectively learn the temporal features in the moving dots experiment. 

Although biological studies confirm the fundamental role of conduction delays in learning \cite{myelin_piano, bio_delay_evidence}, most of the brain-inspired computational models neglect them to avoid any delay-related complexities. However, in this paper, we showed that adding the delays and employing our proposed delay learning rule in a spiking neural network can give a single neuron the ability to learn specific temporal patterns. Adding this property to the spiking neural networks can help to relax delay-related complexities without losing their benefits. Looking from a mathematical point of view, we already know that adding delays to spiking neural networks can result in increasing the dimensionality of those models and thus the advantage of their unprecedented information capacity \cite{polychron}. However, the gap of a practical method for utilizing the delays to use their learning capacity has overthrown their usability. This fact emphasizes the importance of an effective delay learning rule and how it can expand the capability of the future computational models.




\end{document}